\documentclass[conference]{IEEEtran}
\pdfoutput=1
\usepackage{amsmath,amssymb,amsfonts, amsthm}
\usepackage{graphicx}
\usepackage{textcomp}
\usepackage{xcolor}
\def\BibTeX{{\rm B\kern-.05em{\sc i\kern-.025em b}\kern-.08em
    T\kern-.1667em\lower.7ex\hbox{E}\kern-.125emX}}
\usepackage[colorlinks=false, hidelinks]{hyperref}
\usepackage[capitalize]{cleveref}
\usepackage[detect-all,per-mode=symbol]{siunitx}
\usepackage{booktabs}
\newcolumntype{M}{>{$}l<{$}} 
\usepackage{multirow}
\usepackage{pgfplots}
\usepackage[protrusion=true,expansion=true]{microtype} 
\usepackage{algorithm}
\usepackage{algorithmic}
\usepackage{todonotes}

\DeclareMathOperator*{\argmax}{arg\,max}

\newcommand{\KL}{\textrm{KL}}
\newtheorem{lemma}{Lemma}

\usepgfplotslibrary{groupplots}
\usetikzlibrary{positioning}
\pgfplotsset{compat=newest}
\pgfplotsset{every axis legend/.append style={%
cells={anchor=west}}
}
\pgfplotsset{every axis/.append style={
                    label style={font=\footnotesize},
					tick label style={font=\footnotesize},
					legend style={font=\footnotesize}
                    }}

\usepackage[style=ieee, maxcitenames=2, mincitenames=1, backend=bibtex]{biblatex}
\AtEveryBibitem{
 	\clearfield{url}
 	\clearfield{doi}
\ifentrytype{inproceedings}{
 	\clearlist{address}
 	\clearlist{publisher}
 	\clearname{editor}
 	\clearlist{organization}
 	\clearfield{pages}
 	\clearlist{location}
	\clearfield{volume}
	\clearfield{series}
	\clearfield{eprint}
	\clearfield{eprinttype}
 }{}
 }
\AtBeginBibliography{\footnotesize}

\usepackage{float}

\usepackage{comment}

\begin{document}

\title{Model Based Residual Policy Learning with Applications to Antenna Control
}

\author{%
 \IEEEauthorblockN{Viktor Eriksson Möllerstedt} 
 \IEEEauthorblockA{\textit{was with KTH Royal Institute of Technology}
 \\ viktor.mollerstedt@hotmail.com
 }
\and
 \IEEEauthorblockN{Alessio Russo} 
 \IEEEauthorblockA{\textit{Division of Decision and Control Systems} \\ \textit{KTH Royal Institute of Technology} 
 \\ alessior@kth.se
 }
 \and
  \IEEEauthorblockN{Maxime Bouton} 
 \IEEEauthorblockA{\textit{Ericsson Research} 
 \\ maxime.bouton@ericsson.com
 }
}

\maketitle

\begin{abstract}%
Non-differentiable controllers and rule-based policies are widely used for controlling real systems such as telecommunication networks and robots. 
Specifically, parameters of mobile network base station antennas can be dynamically configured by these policies to improve users coverage and quality of service.
Motivated by the antenna tilt control problem, we introduce Model-Based Residual Policy Learning (MBRPL), a practical reinforcement learning (RL) method. MBRPL enhances existing policies through a model-based approach, leading to improved sample efficiency and a decreased number of interactions with the actual environment when compared to off-the-shelf RL methods.
To the best of our knowledge, this is the first paper that examines a model-based approach for antenna control. 
Experimental results reveal that our method delivers strong initial performance while improving sample efficiency over previous RL methods, which is one step towards deploying these algorithms in real networks. 
\end{abstract}


\section{Introduction}\label{sec:introduction}

With the increase in complexity of mobile networks from generation to generation, there has been a growing interest in data-driven method to tune configuration parameters. Networks consist of base station antennas with many configuration parameters that are traditionally configured manually by skilled engineers or using hand-engineered rule-based policies designed to improve network key performance indicators (KPIs) such as coverage, signal quality, or capacity~\cite{mendo2023multi, ai_empow}. 
In contrast, data-driven methods are expected to scale better and adapt to different network conditions.
Reinforcement learning (RL) is one popular and flexible method for automatically learning to tune such parameters from data. 
However, training an RL agent requires a lot of data, which usually involves sampling from an environment. In addition, agents initially tend to have poor performance for multiple iterations before learning useful behaviors. In telecommunications networks, sampling data from the network can be costly, time-consuming, and excessively risky~\cite{rl_challenges}. Addressing sample efficiency, and the poor initial performance of the RL agents, is necessary in order to facilitate their deployment in real networks. 

In this paper, we focus on the problem of tuning base stations' antenna parameters such as the tilt angle, where RL methods have already shown to outperform legacy solutions~\cite{mendo2023multi, ai_empow, safe_tilt}. Previous works have addressed the problem of antenna tilt control with RL by using standard algorithms, such as DQN \cite{dqn}, or focused on the multi-agent aspect~\cite{ai_empow, coord_RL, dandanov2017} without considering sample efficiency. Sample efficiency is important as training RL agents even in simulation can be expensive, and the current number of samples needed to reach a good performance makes online training impractical.

There exist several approaches to address sample efficiency in RL. A practical approach is to use a model of the environment to generate training data for the agent, thus lowering the need for sampling data points from the real environment. Model-based approaches, such as MBPO \cite{MBPO}, and Dreamer \cite{dreamer}, learn such a model using real environment data and use it to generate extra training data for a model-free RL algorithm. These methods have shown impressive gains in sample efficiency, but just as many other RL methods, they suffer from poor initial performance because the agent starts by exploring the environment with random actions. Other methods leverage a baseline policy during training \cite{RPL, dqfd}, which can lead to strong initial performances and, in some cases, an increase in sample efficiency. In that spirit, \textit{Residual Policy Learning} (RPL) \cite{RPL} consists of learning a correction term to a deterministic baseline policy, which does not need to be differentiable.

Another body of literature has investigated safe reinforcement learning approaches for tilting antennas \cite{symbolic_safe_rl_tilt, safety_shield_tilt}. For instance, \citeauthor{safe_tilt} used a rule-based policy as a behavioral policy to gather data from the environment~\cite{safe_tilt}. This data was then used to learn a greedy policy using an off-policy algorithm.
Even though these methods make the training process safer, they  do not necessarily increase the sample efficiency. For antenna tuning problems, sample efficiency is of core importance to enable learning in the real world, but also of practical importance when learning in simulation, since network simulation usually involves expensive calculations. The algorithm proposed in this paper specifically targets this aspect, along with the initial performance of the agent, and can be combined with previous works that consider safety and multi-agent coordination.

With the goal of further increasing sample efficiency while maintaining strong initial performance, we expand upon ideas from both model-based methods and baseline policy techniques. Specifically, we propose a practical model-based RL algorithm that can augment existing hand-engineered policies~\cite{mendo2023multi}, or safe baselines~\cite{safe_tilt} for tuning antenna down tilt in a telecommunication network. However, relying on baselines has drawbacks, mostly due to biased data caused by limited exploration~\cite{safe_tilt, safety_shield_tilt}. By introducing a model-based component, the dataset can be augmented with trajectories generated from an agent that is actively interacting and exploring within the model. This approach seeks to mitigate the bias, allowing for more robust evaluation and potentially enhancing the overall effectiveness of the system.


Contribution-wise, in this work we introduce the problem of controlling antennas in a telecommunication network as an RL problem with continuous actions (discrete actions were used in previous works). 
We present a novel method using a model-based approach to learn a correction term to a baseline policy, extending the RPL idea to stochastic policies, and demonstrate cumulative gain by combining it with model-based methods when applied to controlling antenna tilt.
We provide a theoretical analysis of the algorithm's performance bound and empirically show in simulated mobile networks that our algorithm outperforms the state-of-the-art (DQN~\cite{mendo2023multi}) in terms of sample efficiency and initial performance, making it a more practical alternative for real-world deployment in controlling antennas in a telecommunication network.
Finally, an ablation study highlights the contributions of both our extension to the residual learning concept and the model-based component. Additional details, including proofs of the presented lemmas, simulation parameters, and further experiments, can be found in our technical report {\tt \url{https://arxiv.org/abs/2211.08796}}.

\section{Background}\label{sec:background}

\subsection{Markov decision process}

We model the problem as a Markov Decision Process (MDP): $(\mathcal{S, A}, r, p, \rho_0)$. Here, $\mathcal{S}$ is the state-space, $\mathcal{A}$ the action-space, $r:\mathcal{S}\times\mathcal{A}\times\mathcal{S}\to\mathbb{R}$ the reward function, $p: \mathcal{S}\times \mathcal{A}\to\Delta(\mathcal{S})$ the transition probability function (also known as transition dynamics;  $\Delta(\mathcal{S})$ is the space of probability distributions with support $\mathcal{S}$), and $\rho_0(s)$  is the initial state probability distribution. At step $t$, the agent observes the current state $s_t$ of the system, and selects $a_t$ according to a stationary Markov policy $\pi:\mathcal{S}\to\Delta(\mathcal{A})$. The goal of the agent is to find a policy $\pi$ that maximizes the total discounted reward collected from the environment. For a discount factor $\gamma \in (0,1)$,  we define the discounted value of $\pi$  as
$
    V^\pi(\rho_0) = \mathbb{E}_{s_0 \sim \rho_0} \left[\sum_{t=0}^{\infty} \gamma^t r(s_t, a_t, s_{t+1})  \right]$, where $s_{t+1}\sim p(\cdot|s_t,a_t)$ and $a_t\sim \pi(\cdot|s_t)$.

\subsection{Residual policy learning}

RPL consists in learning a correction term $\pi_c$ to a baseline policy $\pi_b$ (refer to ~\cite{RPL} for more details). The baseline policy $\pi_b$ does not need to be differentiable. This baseline policy can represent prior knowledge in the form of an existing controller. The baseline may come from a hand-engineered method, a control theoretic approach or result from an RL agent trained under different conditions. 

\subsection{Model-based RL} 
In RL problems the transition function is usually unknown, and in model-based RL it is explicitly learned during training by learning a parameter $\theta$ of  a model $p_\theta$ such that $p_{\theta} \approx p$.
Specifically, for a given buffer of experiences $(s_t,a_t,r_t,s_{t+1})\in\mathcal{B}$, a model $p_\theta$ is usually learned by maximizing the log-likelihood of the data so that $\theta \gets \argmax_\theta \mathbb{E}_{(s,a,r,s')\sim \mathcal{B}}\left[\log \mathcal{L}(s,a,r,s';\theta)\right],$ where $\mathcal{L}$ is the likelihood function. Using this learned model, it is possible to use model-free methods to learn a policy $\pi$ using data sampled from $p_\theta$ (for  more details, see also \cite{MBPO}). The benefit of this approach is that the user can significantly reduce the number of experiences sampled from the true environment, which may be challenging or costly to acquire in certain scenarios. Finally, note that  the reward function is often assumed to be known, and other times it must be learned along with the transition dynamics. In this work, we consider both cases.
\section{Problem Formulation}\label{sec:problem}
Mobile telecommunication networks are composed of a number of \textit{base stations} to which one or several antennas are mounted. The antennas transfer data to and from several users, such as cellphones and computers. A user decides which antenna to attach to based on signal strength. Users attached to the same antenna form a \textit{cell}. In the coverage and capacity optimization problem, the goal is to control the tilt of the antenna such that all users have good coverage, good signal quality, and that many users can send and receive data at the same time. When adjusting the parameters of the antennas, these  quantities will be affected, and the down-tilt angle $w$ is the one of the most influential parameters \cite{self_tuning_tilt, safe_tilt, opt_cco_ml}. \cref{fig:antenna_env} (right) illustrates its influence on coverage and signal quality.

\begin{figure}[t]
    \centering
    \input{figures/antenna_env}
    \includegraphics[width=0.35\columnwidth]{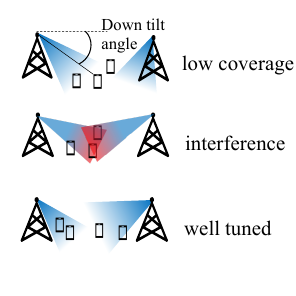}
    \caption{\textbf{Left:} aerial view of the simulated network. 3 antennas are attached to each base station, pointing in different directions. \textbf{Right:} illustration of the effect of changing the tilt angle on coverage and quality.}
    \label{fig:antenna_env}
\end{figure}

\subsection{System Model}\label{sec:model}

A common way to measure coverage is via the Reference Signal Received Power (RSRP), which is the power of the signal received by a user attached to a cell \cite{opt_cco_ml, ai_empow}. We denote the RSRP of user $u$ attached to cell $c$ as $\rho_c^u$.  The RSRP is a function of the transmitted power from the antenna $P_c$, the gain of the antenna $G_c^u$, and path loss $L_c^u$: $\rho_c^u = P_c G_c^u L_c^u$. The gain is a function of the down-tilt angle $w$, and the path-loss depends on obstacles (such as buildings and trees) and the medium of transmission. 
To model the relation between tilt angle and the antenna gain, we use a standardized horizontal vertical radiation pattern according to 3 GPP case 1 and 3~\cite{3gpp.36.814}. 

Since RSRP is a function of $w$ and the users decide which cell to attach to based on the RSRP, the down-tilt angle affects how many users are attached to a cell. There are a number of ways to use RSRP to measure the coverage of the entire cell. We used the average log RSRP across users attached to that cell: $COV_c = \frac{1}{|U_c|} \sum_{u \in U_c} \log \rho_c^u$, where $U_c$ is the set of indices of users in cell $c$. 

 The signal quality can be measured via the Signal to Interference and Noise Ratio (SINR). The SINR $\gamma$ of user $u$ attached to cell $c$ can be defined as the corresponding RSRP value, divided by a noise term plus the RSRP from all other cells: $\gamma_c^u = \frac{\rho_c^u}{\kappa + \sum_{i \in C \backslash c } \rho_i^u }$.
Here, $\kappa$ is the noise term and $C$ the set of cell indices. We measured the quality in a cell as the average log SINR over users: $QUAL_c = \frac{1}{|U_c|} \sum_{u \in U_c} \log \gamma_c^u$. \\

The throughput for user $u$ attached to cell $c$ can be defined as: $T_c^u = \frac{\omega_B n_B}{|U_c|} \log_2(1 + \gamma_c^u)$.
Here, $\omega_B$ is the bandwidth per physical resource block, and it is assumed that each user is assigned the same number of physical resource blocks $n_B$. We measure the capacity for a cell with the average log throughput across users: $CAP_c = \frac{1}{|U|_c} \sum_{u \in U_c} \log T_c^u$. \\


\subsection{MDP Formulation}

The antenna tilt problem can be modeled as an MDP where each antenna is controlled by one agent. It is then inherently a multi-agent problem, since there are many antennas interacting in the same environment. 
We approached the problem of controlling multiple antennas via parameter sharing. 
Our MDP formulation is a single-agent formulation where a shared policy and model is trained to represent the environment from a point of view of a single antenna (not a global policy and model). They are still trained in an environment with multiple antennas and the data generated by all the agents is used for training the model and policy. In other words, one step in the environment provides $n$ transition samples where $n$ is the number of agents. This is equivalent to saying that all antennas are represented with a common MDP for each antenna and are controlled with the same policy.
Other lines of work address the problem of coordinating antennas ~\cite{coord_RL}, and combining those approaches with our algorithm is left for future work. The MDP is formulated as follows from the point of view of a single antenna. 

\paragraph{Observation space} The agent observes its current tilt angle $w$, and the current value of the key performance indicators (KPIs): coverage, capacity and quality in its corresponding cell. The dimensionality of the observation space is $4$. 

\paragraph{Action space} The agent outputs one continuous change of tilt-angle $(\Delta w)_i \in  [\SI{-1}{\degree}, \SI{1}{\degree}]$ for each antenna $i$. The tilt-angle  $w_i$ for the $i$-th antenna is limited to lie between $\SI{0}{\degree}$ and $\SI{15}{\degree}$. The change in tilt angle affects the KPIs according to the model described in \ref{sec:problem}.

\paragraph{Reward function} The reward function at time $t$ is a sum of the coverage, capacity and quality at time $t+1$. For a cell $c$ we have: $r_{c,t} = COV_{c,t+1} + QUAL_{c,t+1} + CAP_{c,t+1}$. 

All three metrics were normalized to have mean $\mu \approx 0$ and standard deviation $\sigma \approx 1$. We had to limit the standard deviation further for some of the metrics to prevent outlier data. The mean and standard deviation were measured empirically by running simulations with a random policy prior to training the agent. As explained in \cref{sec:model}, we use the logarithmic of the geometric mean to compute the KPIs in order to provide some notion of fairness in the reward function. It discourages giving a few users very bad KPI values in order to increase the majorities values.\\

\section{Method}\label{sec:method}
We are interested in solving the problem by reducing the number of interactions needed with the actual environment, which is a common requirement in many applications. To that aim, we propose Model-Based Residual Policy Learning (MBRPL), which improves on the residual-policy learning concept  by considering a model-based approach. Firstly, we incorporate prior knowledge through a baseline policy to achieve strong initial performance and to guide the training in a sound direction. Secondly, we use a model-based approach for training the correction term to reduce the number of samples needed from the true environment.



\subsection{Stochastic Residual Policy Learning}
As explained in \cref{sec:background}, Residual Policy Learning  combines a baseline policy $\pi_b$ with a correction term $\pi_c$. The baseline policy does not need to be differentiable, and can be of any form. \citeauthor{RPL} \cite{RPL} consider deterministic policies, whereas we focus on stochastic policies (such as PPO \cite{schulman2017proximal} or SAC \cite{SAC_v1})  that make use of an actor-critic training procedure. In fact, these algorithms have empirically shown to lead to more stable training than deterministic policies trained, for example, using DDPG.
Combining a stochastic policy $\pi_c$ with a baseline policy $\pi_b$, which does not necessarily need to be stochastic, can be done in different ways depending on the problem of interest. In this work, we  focus on problems with continuous action spaces, and therefore at step $t$ the action chosen by the agent can be represented as $a_t = f(a_t^c, a_t^b)$, where $(a_t^c, a_t^b)$ are, respectively, the actions chosen by the correction term and the baseline term at time $t$. The function $f$ combines the two actions and can be customized.

For example, assuming that the correction term is represented by a Gaussian distribution of parameters $\phi=(\mu,\sigma)$. If the baseline term is deterministic, and $f(x,y)=x+y$, then the overall policy at time $t$ can be expressed as
$
    \pi_\phi(\cdot|s_t) = \mathcal{N}\left(\mu(s_t) + a_t^b, \sigma^2(s_t)\right),
$
where the parameters $(\mu,\sigma)$ are learned online using classical policy learning methods, such as SAC.
If the baseline policy also represents a Gaussian distribution with parameters $\mathcal{N}(\mu_b,\sigma_b^2)$, independent of the correction term, then we simply derive
$
    \pi_\phi(\cdot|s_t) = \mathcal{N}\left(\mu_b(s_t) + \mu(s_t), \sigma_b^2(s_t) + \sigma^2(s_t)\right).
$ The policy is initialized to closely follow the baseline $\pi_b$. For instance, if the correction term is a neural network, we initialize the weights of the last layer to be approximately $0$. This initialization leads to a stronger initial performance. However, the random initialization of the critic can still create an initial performance drop, since the critic guides the training of the actor in the "wrong" direction. To overcome this problem, just as in RPL, we let the critic train while keeping the policy unchanged during the initial phase. The number of training steps during which the critic trains with a constant policy is denoted as Critic Burn-In (CBI), represented by the parameter $B_{in}$.
In summary, this approach extends the existing RPL algorithm to support training stochastic policies and use more recent model-free RL algorithms such as SAC to learn the policy residual. 

\subsection{Model-Based Residual Policy Learning (MBRPL)}

\begin{algorithm}[t]
\caption{Model-Based Residual Policy Learning (MBRPL)}
\begin{algorithmic}[1]\label{alg:mbrpl}
    \REQUIRE Baseline policy $\pi_b$; critic burn-in period $B_{in}$.
    \smallskip
    \STATE  Initialize model $p_{\theta}$;   critic $Q_{\psi}$ and replay buffer $\mathcal{B}$.
    \STATE  Initialize combined policy $\pi_\phi$, where $\phi$ is the parameter of the correction term.
    \FOR{$t=1,\dots,T$}
        \STATE Sample experiences $(s_t,a_t,r_t,s_{t+1}\dots)$ from the true environment using $\pi_{\phi}$ and add them to the buffer $\mathcal{B}$.
        \STATE Train $p_\theta$ on a batch $B$ sampled from $\mathcal{B}$ using maximum likelihood.
        \STATE Sample multiple experiences $(\mathbf{s}_\delta,\mathbf{a}_\delta,\mathbf{r}_\delta,\mathbf{s}_{\delta+1})$ from $\mathcal{B}$.
            
        \FOR{$\tau=\delta,\dots,\delta+H$} 
            \STATE $\mathbf{s}_{\tau + 1} \sim p_{\theta}(\mathbf{s}' \mid  \mathbf{s}_{\tau}, \pi_{\phi}(\mathbf{s}_{\tau}))$ \algorithmiccomment{Predict next state batch}
            \STATE $\mathbf{r}_{\tau} = r(\mathbf{s}_{\tau}, \mathbf{a}_{\tau}, \mathbf{s}_{\tau + 1})$
            \STATE Compute critic loss on $(\mathbf{s}_{\tau}, \mathbf{a}_{\tau}, r_\tau, \mathbf{s}_{\tau + 1})$ and update $\psi$ using gradient descent. 
            \IF{$t> B_{in}$  \{CBI condition\}}
                \STATE Compute the actor loss on $(\mathbf{s}_{\tau}, \mathbf{a}_{\tau}, r_\tau, \mathbf{s}_{\tau + 1})$ and update the parameter $\phi$ using gradient descent.
            \ENDIF
            
        \ENDFOR
    \ENDFOR
    \end{algorithmic}
\end{algorithm}

Motivated by maximizing sample efficiency, we propose to combine stochastic RPL with a model-based approach. We give a high level description of the algorithm in \cref{alg:mbrpl} (the algorithm describes an on-policy training, but it can be easily adapted to be off-policy). The training method consists of alternating between learning the dynamics, predicting future states and rewards, and training the residual policy and critic on the real and predicted data. 

A function approximator $p_\theta$ is introduced to model the true environment, which is trained by maximizing the likelihood between between the generated data and true data sampled from the environment (which reduces to the classical MSE criterion for Gaussian transition functions). The frequency at which the model is trained and the number of data points used to train it at each step are hyperparameters of the algorithm. This part of the method is similar to existing model-based RL methods~\cite{MBPO}. 

The policy is formed using stochastic residual policy learning, and the correction term is trained using  trajectories generated by the learned model $p_{\theta}$ using off-policy model free methods (however, also on-policy methods can be used). In particular, we use SAC to train the stochastic policy residual. A specificity of our method is that we perform a policy update after each generated trajectory points instead of considering the whole trajectory as a batch. All the hyperparameters and design choices (e.g. model representation) will be discussed in the experiment section and appendix.

\subsection{Theoretical performance}
We now theoretically analyze the performance bound of the learned policy compared to both how close the baseline policy is to the optimal solution and how close the model is to the true transition model. Since the baseline policy is not necessarily created using the same environment $M$ where the corrected policy will operate on, but possibly a different one $M_b$, we analyzed the performance of the corrected policy in $M$ depending on its performance in $M_b$. To that aim, for a discount factor $\gamma$, we denote by $V_M^\pi$ the discounted value of $\pi$ in $M$, and similarly we indicate by $V_{M_b}^\pi$ the discounted value of $\pi$ in $M_b$.

To derive a performance bound, we first consider the following lemma that bounds the performance of a generic policy $\pi$ in two similar environments $M$ and $M_b$, with the same reward, and different transition functions $p_0$ and $p_1$. Assuming that $(p_0,p_1)$ are similar in the Kullback-Leibler sense, we derive the following result (the proof is provided in the appendix).
\begin{lemma}\label{lemma:value_model_difference}
Consider two MDPs $M=({\cal S}, {\cal A}, r, p_1)$ and $M_b=({\cal S}, {\cal A}, r,p_0)$ and a Markov stationary policy $\pi$. Let $r\in [0,1]$, and assume that  $\KL(p_0(s,a),p_1(s,a))=\mathbb{E}_{s'\sim p_0(\cdot|s,a)}\left[\log \frac{p_0(s'|s,a)}{p_1(s'|s,a)}\right]\leq \varepsilon$ for all $(s,a)$. Then $|V_{M_b}^{\pi}(\mu) - V_{M}^{\pi}(\mu)| \leq \sqrt{2\varepsilon}\frac{\gamma}{1-\gamma} \|V_{M_b}^\pi\|_\infty$ for any  distribution $\mu$ of the initial state.
\end{lemma}
Given this performance bound, the idea is to evaluate the performance of the corrected policy $\pi_{\phi}$ knowing the performance of the baseline term $\pi_{b}$ in the environment $M_b$ in which it was trained. Similarly as before, assuming that the correction policy and the baseline policy are close in the KL-sense, we derive the following.
\begin{lemma}\label{lemma:performance_bound}
Consider two MDPs $M=({\cal S}, {\cal A},r,p)$ and $M_b=({\cal S}, {\cal A},,r,p_0)$, with $r\in[0,1]$, that satisfy $\KL(p_0(s,a),p(s,a))\leq \varepsilon_0$ for all $(s,a)\in S\times A$. Let $\pi_b$ be a Markov policy trained on  $M_b$, and let its average total discounted reward be $V_{M_b}^{\pi_b}(\mu)$, for some initial distribution of the state $\mu$ and discount factor $\gamma$. Assume that $\max_s \KL(\pi_\phi(s),\pi_b(s))\leq \varepsilon_\pi$. Then
\begin{equation}
V_{M}^{\pi_\phi}(\mu) \geq 
 V_{M_b}^{\pi_b}(\mu)- \frac{\sqrt{2}}{1-\gamma}\left(  \frac{\sqrt{\varepsilon_\pi}}{(1-\gamma)} + \gamma\sqrt{\varepsilon_0} \|V_{M_b}^{\pi_b}\|_\infty\right)
\end{equation}
\end{lemma}

Lemma \ref{lemma:performance_bound} tells us that if the two policies are similar enough, the performance of the corrected policy in $M$ is comparable to that of the baseline in $M_b$, if the two environments are not too different. This result motivates learning a correction term to the baseline policy in an environment in which the baseline performs suboptimally. Initializing the correction term to $0$ at the beginning of training encourages similarity of the corrected policy and the baseline. In the next section, we empirically demonstrate the strength of this approach. 


\section{Experiments}\label{sec:experiments}


In this section, we empirically evaluate our proposed model-based RL method by applying it to a realistic telecommunication network problem. Specifically, we focus on optimizing coverage and capacity through the control of antenna tilt angles."

We examined several questions: (1) whether MBRPL was generally more sample-efficient than existing methods; (2) whether it could maintain a strong initial performance; and (3) whether both the residual and model-based components contribute to a performance improvement.

We compared MBRPL against several well-known model-free baselines and performed an ablation study. We omitted existing expert-based methods from this comparison, as they have already been shown to be significantly outperformed by one of our baselines (DQN) in previous works~\cite{safe_tilt, coord_RL, safety_shield_tilt}. Additionally, we investigated how the critic burn-in affects MBRPL's initial performance.

\subsection{Compared Methods}

We compared the sample efficiency and performance at convergence of MBRPL against several baselines. SAC is a model-free state-of-the-art RL algorithm~\cite{SAC_v1, SAC_v2}. DQN \cite{dqn} is a well-known discrete action-space model-free method used in several previous works on RL for antenna tilt control which has been shown to outperform non-RL baselines \cite{safe_tilt,coord_RL, safety_shield_tilt}. For DQN, the action space is changed to update the tilt by discrete increments of $\{\SI{-1}{\degree},\SI{0}{\degree}, \SI{1}{\degree} \}$ compared to the other methods performing continuous increments.
We also compared to ablations of our method: model-based SAC (MBSAC) and stochastic RPL (SRPL). MBSAC uses only the model-based part of the algorithm to train a SAC agent. SRPL learns a residual policy to a stochastic baseline policy using SAC.
Finally, we study the effect of the critic burn-in parameter on MBRPL. This parameter controls how many steps the policy is frozen at the baseline policy while the critic trains.

\paragraph{Baseline policy.} As baseline policies for MBRPL and SRPL we experimented with two versions, both SAC agents trained in modified versions of the environment (MDP $M_b$ in our theoretical analysis):
\begin{itemize}
    \item An environment where all buildings were removed. Buildings affect how signals propagate~\cite{axcel}.
    \item An environment where the intersite distance between base-stations was reduced to 400 meters. A shorter distance intuits that a larger down-tilt is needed to optimize performance. This baseline has a worse performance than the one above in the training environment. 
\end{itemize}
The baseline policies resulting from our approach outperformed a random policy, but exhibited suboptimal performance within the true environment. These policies operated within the same observation and action space as our agent. Although the baseline trained with SAC is a stochastic policy, we treated it deterministically by relying solely on $\mu_b$, that is $a_t=\mu_b(s_t)$. Following this initial training, the baseline policy was not subject to further refinement. It's worth noting that other methods, such as classical control or rule-based techniques, could also be employed to generate these baseline policies.

Finally, the transition model used by MBRPL and MBSAC is a neural network that outputs the mean and variances of a Gaussian distribution. We found out that a single model was  sufficient to learn an efficient policy in the antenna environment, rather than using an ensemble as in MBPO~\cite{janner2019trust}. 
The details of the hyperparameters can be found in \cref{tab:hyperparameters} or in our technical report. For the MBSAC baseline, we used the same model training hyperparameters as in MBRPL.

\subsection{Simulation and Training}

The environment was simulated using a proprietary system level mobile network simulator relying on a map-based propagation model to compute the signal received by each user \cite{axcel}. The network was built as a hexagonal grid of 7 base stations with 3 antennas each (\num{21} agents), with parameters corresponding to the standardized 3GPP case 1~\cite{3gpp.36.814}, with \SI{500}{meter} intersite distance. 1000 static users were uniformly distributed across the environment, see \cref{fig:antenna_env} (left). The environment is a 50-50 split of indoor and outdoor environment, with buildings placed uniformly at random across the map, which consists of \num{5000} discrete square bins. Buildings affects how the signal propagates between the antenna and a user as detailed in \cite{axcel}. They are excluded from the figure to prevent cluttering. 

We let each method train for \num{10000} steps across $5$ random seeds. The positions of the users and buildings were uniformly randomly generated at the start of each episode, and the tilt of the antennas were initialized uniformly at random within the allowed range. 
For one step of the environment, we collect a transition sample from all \num{21} antennas and add them to the replay buffer.
We began by tuning the hyperparameters of the benchmark algorithms (SAC, DQN), and then used the same parameters for MBRPL and the ablations. Our methods require some additional settings, such as choice of baseline policy, critic burn-in and prediction horizon. Details of the hyperparameters can be found in \cref{tab:hyperparameters}.

\subsection{Numerical Results}

\begin{figure*}
    \centering
    \input{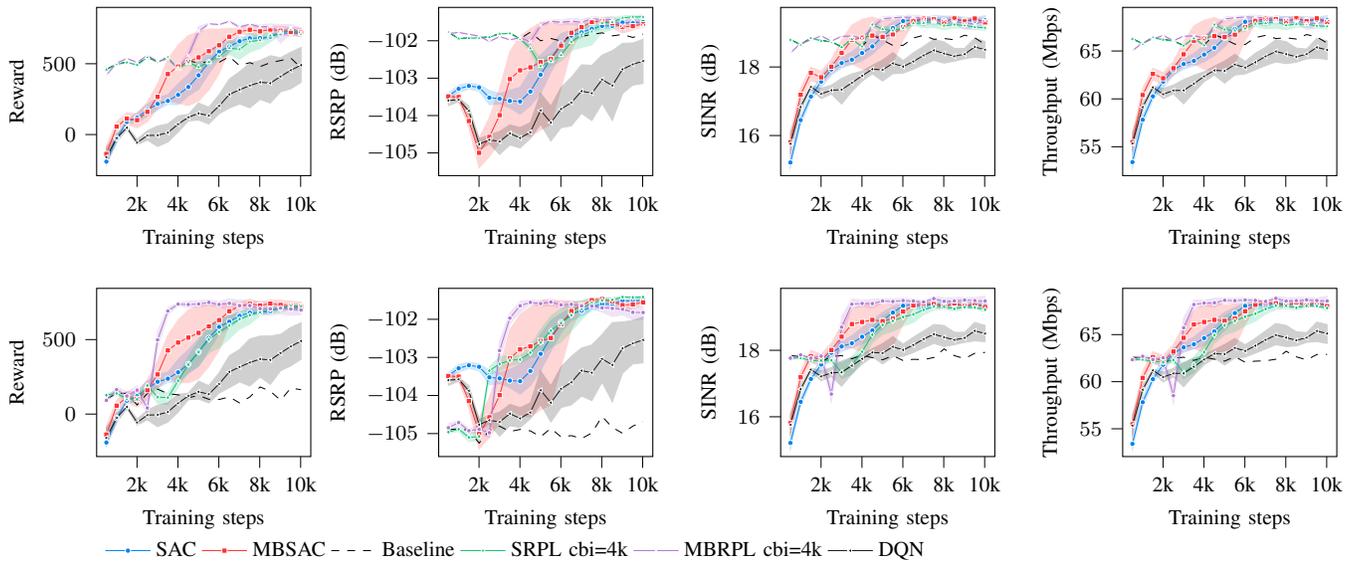}
    \caption{MBRPL training performance compared to the benchmarks SAC, DQN, and the ablations MBSAC and SRPL. The agents are trained in an environment with buildings and an intersite distance of \SI{500}{\meter}. In the top row, we use a baseline policy trained in a completely outdoor environment, in the bottom row, a baseline policy trained with \SI{400}{\meter}. The baseline policy is a fixed (non-learning based) policy. The shaded area represents the \SI{95}{\percent} confidence interval.}
    \label{fig:mbrpl-comp-both-baselines}
\end{figure*}

\begin{figure}
    \centering
    \input{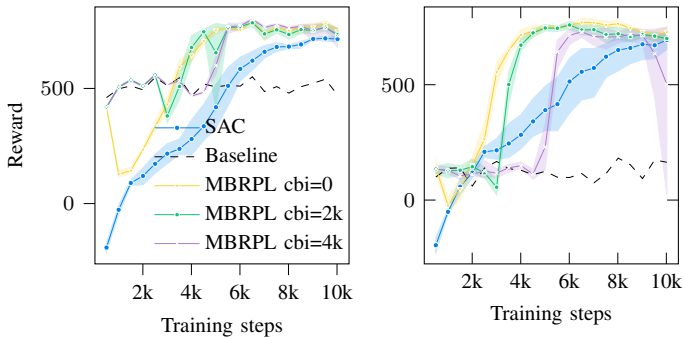}
    \caption{The effect of burn-in on MBRPL. SAC is also plotted for comparison. (left) expert trained in an outdoor environment. (right) expert trained in environment with intersite distance \SI{400}{\meter}.The shaded area represents the \SI{95}{\percent} confidence interval.}
    \label{fig:cbi-mbrpl-both-baselines}
\end{figure}
 A comparison of MBRPL to the benchmarks and the ablations can be seen in \cref{fig:mbrpl-comp-both-baselines}. We show the collected reward, RSRP, SINR and throughput during training for each method.

In the top row, MBRPL and SRPL were trained using the outdoor baseline policy. Here MBRPL is the most sample efficient by far, converging at around \num{5500} steps in terms of reward. MBSAC converges at around \num{7500} steps, and SAC and SRPL at around \num{9000} steps. DQN did not converge before the step limit was reached. A similar trend is observed for the RSRP plot. In terms of SINR and throughput, SAC and SRPL converge faster than they did in terms of RSRP, rivaling MBSAC. MBRPL consistently converges the fastest at around \num{5500} steps.

The baseline policy reached an average reward of $682 \pm 57$ (1 std) in the environment without buildings, on which it was trained. In the evaluation environment of \cref{fig:mbrpl-comp-both-baselines} (top row) it collected an average reward of $508 \pm 26$ (1 std). 

In the bottom row of figure \cref{fig:mbrpl-comp-both-baselines}, MBRPL and SRPL were trained using the baseline with a shorter intersite distance. Here we observe similar results: MBRPL is the most sample efficient. This time it converges earlier in terms of reward: after around \num{4000} steps. 
The baseline policy reached an average reward of $1599 \pm 66$ (1 std) in the environment with shorter intersite, on which it was trained. In the true environment, it only collected an average reward of $126 \pm 33$ (1 std).
Interestingly, this baseline policy performed much worse than the baseline trained on an outdoor environment. Despite the bias from the baseline, MBRPL is still able to find a policy as good as the SAC policy while using less samples.

MBSAC is not significantly more sample efficient on average than SAC, in spite of using 10 times as much data to learn the policy. There is a large variance in performance across the different seeds. 
Since SAC is not having similar issues, it is reasonable to assume that the model-based component is causing it, and may be due to poor model accuracy. Epistemic uncertainty in model prediction could be addressed by using an ensemble of transition models, as in MBPO.
Both MBRPL and SRPL have a  stronger initial performance due to the baseline policy, although SRPL is not more sample efficient than SAC. This highlights the strength of the combined approach of MBRPL. The strong initial policy allows more efficient use of the learned model, increasing sample efficiency significantly.

Figure \ref{fig:cbi-mbrpl-both-baselines} shows the effect of changing the critic burn-in parameter. In the left plot, the baseline policy was trained in an environment without buildings. A low burn-in value leads to a large initial dip in performance, because of the mismatch between actor and critic performance, similar to what was observed in the original RPL paper~\cite{RPL}. MBRPL is able to recover quickly, and still converges well before SAC. Increasing the critic burn-in reduces the dip, but can delay convergence because the policy training is postponed. A similar trend is observed when using the baseline policy trained with a smaller intersite distance, see \cref{fig:cbi-mbrpl-both-baselines} (right). Here we also observe that setting the critic burn-in value to \num{2000} lead to instabilities during training for one of the seeds, resulting in worse performance at the end of training. Comparing the left and right plots show that the initial performance of the baseline policy does not prevent the algorithm from converging to a good final policy in any of the burn-in settings.


\begin{table}[h]
    \centering
\resizebox{\columnwidth}{!}{%
    \begin{tabular}{lM}
    \toprule
    \multicolumn{2}{l}{\textbf{SAC Parameters}} \\
    Actor and Critic MLP   & [64, 64, 64], \text{ReLU activation}               \\ 
    Actor and Critic lr & 3 \cdot 10^{-4} \\
    Buffer size           & 10000                      \\ 
    Policy distribution   & \text{Tanh squashed Gaussian} \\ 
    $\gamma$              & 0.9                       \\ 
    $\tau$                & 5\cdot 10^{-3}                      \\ 
    Target network update & \text{Every other time step}             \\ 
    Batch size            & 128                         \\ 
    Entropy lr            & 3 \cdot 10^{-4}               \\ 
    $\alpha_0$              & 1                      \\ 
    Target entropy        & -1    \\ 
    \midrule
    \multicolumn{2}{l}{\textbf{SRPL Parameters (same as above plus below)}} \\
    $B_{in}$ & \text{0, 2k, or 4k} \\
    baseline & \text{outdoor agent or 400m ISD agent} \\
    \midrule
    \multicolumn{2}{l}{\textbf{MBRPL Parameters (same as above plus below)}} \\
    Transition model & [64, 64, 64], \text{ReLU activation} \\
    Model lr & 10^{-3} \\
    H & 10 \\ 
    rollout batch size & 128 \\
    \midrule 
    \multicolumn{2}{l}{\textbf{Simulation parameters}} \\
    Antenna height & \SI{32}{\meter} \\
    Antenna model & \text{HV 3gpp 36.814} \\
    Max Tx power & \SI{40}{\watt} \\
    Frequency & \SI{2}{\giga\hertz} \\ 
    \bottomrule
    \end{tabular}
}   
    \caption{Hyperparameters for MBRPL and its ablations, and other relevant simulation parameters.}\label{tab:hyperparameters}
    \end{table}

\section{Conclusion}\label{sec:conclusion}
In this work, we have presented a model-based RL method which learns a residual correction term to a baseline policy. 
Our method proved effective for optimizing coverage and capacity on an antenna tuning problem. Model-based RL had not previously been tested on this problem, and our method outperformed all benchmarks.
Ablation studies show that combining the usage of a baseline policy with a model-based approach leads to higher sample efficiency. A hyperparameter study indicates that the higher initial performance of the baseline policy can be maintained by setting the appropriate critic burn-in, however, further testing is needed to confirm this result. Our results hint that relying on an existing suboptimal controller, paired with a model-based approach, is a viable approach for deploying intelligent control algorithms in real-world applications. A limitation of our work is that we simplified the multi-agent nature of the problem through parameter sharing. Adding coordination mechanisms from related work to model-based algorithms would be an interesting future direction. In this paper, we have focused on LTE networks (4G), but we believe that our results would also apply for future generation mobile networks. An interesting future research project would therefore be to apply this algorithm to tilt control for 5G urban macro. Other future research directions could involve model residual learning when approximate models are available, as well as extending RPL to discrete action spaces.

\printbibliography

\newpage
\onecolumn
\appendix
\section*{Appendix A. }
This appendix contains the proof of lemma \ref{lemma:value_model_difference}.
\begin{proof}[Proof of Lemma \ref{lemma:value_model_difference}]
 For the sake of notation, let $V_0 = V_{M_0}^\pi$ and $V_1 = V_{M_1}^\pi$. Further, define $\Delta V(s) = V_{0}(s) - V_{1}(s)$, $\Delta p(s'|s,a) = p_0(s'|s,a)-p_1(s'|s,a)$.
\begin{align*}
    \Delta V(s) &= \mathbb{E}_{a\sim\pi(s)} \Big[ \gamma \mathbb{E}_{s_0'\sim p_0(s,a)}[V_0(s_0')]-\gamma \mathbb{E}_{s_1'\sim p_1(s,a)}[V_1(s_1')]\Big],\\
    &= \gamma  \mathbb{E}_{a\sim\pi(s)} \Big[\mathbb{E}_{s_0'\sim p_0(s,a)}[V_0(s_0')] - \mathbb{E}_{s_1'\sim p_1(s,a)}[V_1(s_1')\pm V_0(s_1')]\Big],\\
    &= \gamma \mathbb{E}_{a\sim\pi(s)} \Big[\underbrace{\left(\mathbb{E}_{s_0'\sim p_0(s,a)}[V_0(s_0')]-\mathbb{E}_{s_1'\sim p_1(s,a)}[V_0(s_1')]\right)}_{(a)} + \underbrace{\mathbb{E}_{s_1'\sim p_1(s,a)}[\Delta V(s_1')]}_{(b)}\Big].
\end{align*}
By expanding $(b)$ recursively, we find that
\begin{align*}
\Delta V(s) &= \frac{\gamma}{1-\gamma}\mathbb{E}_{z\sim \mu_1^\pi(s), a\sim \pi(z)}\Big[\mathbb{E}_{s_0'\sim p_0(z,a)}[V_0(s_0')]-\mathbb{E}_{s_1'\sim p_1(z,a)}[V_0(s_1')] \Big],
\end{align*}
where $\mu_1^\pi(s)$ is the discounted policy distribution induced by $\pi$ starting in $s$ in model $M_1$.
By Pinsker’s inequality we bound (a) as follows:
\begin{align*}
&\left|\mathbb{E}_{s_0'\sim p_0(s,a)}[V_0^\pi(s_0')]-\mathbb{E}_{s_1'\sim p_1(s,a)}[V_0^\pi(s_1')]\right|\leq \|p_0(s,a)-p_1(s,a)\|_1 \|V_0\|_\infty.
\end{align*}
Consequently, we have
$
\left|\Delta V(s)\right| \leq  \frac{\gamma\sqrt{2}}{1-\gamma} \|V_0\|_\infty\mathbb{E}_{z\sim \mu_1^\pi(s), a\sim \pi(z)}\left[ \sqrt{\KL(p_0(z,a), p_1(z,a))}\right].
$
\end{proof}

\begin{proof}[Proof of Lemma 2]
Note that $|V_{M}^{\pi_\phi}(\mu) - V_{M_b}^{\pi_b}(\mu)| =  |V_{M}^{\pi_\phi}(\mu) - V_{M_b}^{\pi_b}(\mu) \pm V_{M}^{\pi_b}(\mu)| \leq  |V_{M}^{\pi_\phi}(\mu) - V_{M}^{\pi_b}(\mu)| + |V_{M_b}^{\pi_b}(\mu)- V_{M}^{\pi_b}(\mu)|$. The second term can be bounded using Lemma \ref{lemma:value_model_difference} as $|V_{M_b}^{\pi_b}(\mu)- V_{M}^{\pi_b}(\mu)| \leq \sqrt{2\varepsilon_0}\frac{\gamma}{1-\gamma} \|V_{M_b}^{\pi_b}\|_\infty$. The first term can be bound as in \cite[Lemma B.3]{janner2019trust}, i.e., $|V_{M}^{\pi_\phi}(\mu) - V_{M}^{\pi_b}(\mu)| \leq \sqrt{2\varepsilon_\pi} \frac{1}{(1-\gamma)^2}$, where we also made use of Pinsker's inequality. We conclude that
$
|V_{M}^{\pi_\phi}(\mu) - V_{M_b}^{\pi_b}(\mu)| \leq   \frac{\sqrt{2 \varepsilon_\pi}}{(1-\gamma)^2} +\sqrt{2\varepsilon_0}\frac{ \gamma}{1-\gamma} \|V_{M_b}^{\pi_b}\|_\infty,
$ from which the result follows.
\end{proof}

\section*{Appendix B.}

This section provides details on the model and policy training steps that we used for MBRPL in the antenna tilt problem. During training, these two steps were performed for each step in the real environment. See \cref{alg:mbrpl} for more details.

\subsection*{B.1. Model learning step}

We trained a single model using a mean squared error supervised loss. Firstly, under the Markov assumption:
\begin{align}
   p(s_{0:T} \mid a_{0:T-1}) = \rho(s_0) \prod_{t=1}^{T-1} p(s_{t+1} \mid s_t, a_t).
\end{align}
Furthermore, we assume that $p_{\theta}(s_{t+1} \mid s_t, a_t)$ is normally distributed:
\begin{equation}
    p_{\theta}(\cdot \mid s_t, a_t) = \mathcal{N}(\mu_{\theta}(s_t, a_t), \Sigma_{\theta}(s_t, a_t)),
\end{equation}
and we use a dense neural network parameterized by $\theta$ to output the mean $\mu_{\theta}$ and the diagonal co-variance $\Sigma_{\theta}$. The loss is calculated using a MSE between a batch of the normalized \textit{true} states $\mathbf{\hat{s}}$ and the mean of 10 batches of the normalized \textit{generated} states $\mathbf{\Bar{s}}_{gen}$:
\begin{equation}
    L_{\theta} = \frac{1}{B\cdot d}\sum_{b=0}^{B-1} \mid\mid \mathbf{\hat{s}}^{(b)} - \mathbf{\Bar{s}}_{gen}^{(b)} \mid\mid_2^2 
\end{equation}
Here, $\mathbf{\hat{s}}^{(b)}$ indicates the $b$-th state vector in the batch, and similarly for $\mathbf{\Bar{s}}_{gen}^{(b)}$. The states are re-scaled by element-wise division with $\mathbf{s}_{max} - \mathbf{s}_{min}$, which is the difference between the upper and lower bounds of the state-space. Also note that $d$ is the dimension of the state-space and $B$ the batch size. The parameter $\theta$ is then learned via gradient descent using the re-sampling trick on the generated data. One model learning step is performed on one batch of data for each step in the true environment.

\subsection*{B.2. Policy learning step}

The residual policy is trained using generated data from the learned transition model $p_\theta(s' \mid s,a)$. $H$ actor-critic updates are made per sample batch from the buffer of real experiences. We used the actor-critic loss functions of the SAC method \cite{SAC_v2}. Pseudocode for a single policy update step can be seen in \cref{alg:policy_learning_step}. One such learning step is performed on a batch of data sampled from the buffer for each step in the real environment.

\begin{algorithm}[H]
\caption{Policy learning step}
\begin{algorithmic}[1]
\label{alg:policy_learning_step}

\REQUIRE Batches $(\mathbf{s}_t, \mathbf{a}_t, \mathbf{r}_t, \mathbf{s}_{t+1})$, transition model $p_{\theta}$, baseline policy $\pi_b$, residual policy $\pi_{\phi}^b$, critic $Q_{\psi}$, reward function $r$

\FOR{$\tau$ from $t$ to $t + H$}

    \IF{$\tau > t$}
        \STATE Sample action batch $\mathbf{a}_{\tau}$ using  $\pi_{\phi}^b(\mathbf{s}_{\tau}) $
        \STATE $\mathbf{s}_{\tau + 1} \sim p_{\theta}(\mathbf{s}' \mid  \mathbf{s}_{\tau}, \mathbf{a}_{\tau})$
        \algorithmiccomment{Predict next state batch}
        \STATE $\mathbf{r}_{\tau} = r(\mathbf{s}_{\tau}, \mathbf{a}_{\tau}, \mathbf{s}_{\tau + 1})$
    \ENDIF
        
    \STATE Perform actor-critic learning step on batches ($\mathbf{s}_{\tau}, \mathbf{a}_{\tau}, \mathbf{r}_{\tau}, \mathbf{s}_{\tau + 1}$)\

\ENDFOR

\end{algorithmic}

\end{algorithm}

\section*{Appendix C. Robotic control}

We examine the sample efficiency and performance at convergence of MBRPL on several control tasks simulated with the MuJoCo physics engine. The problems we study are Hopper, Walker (\cite{deepmind_ctrl}), and Ant with truncated observations. 

\subsection*{C.1. Experimental setup}
\begin{figure*}[!ht]
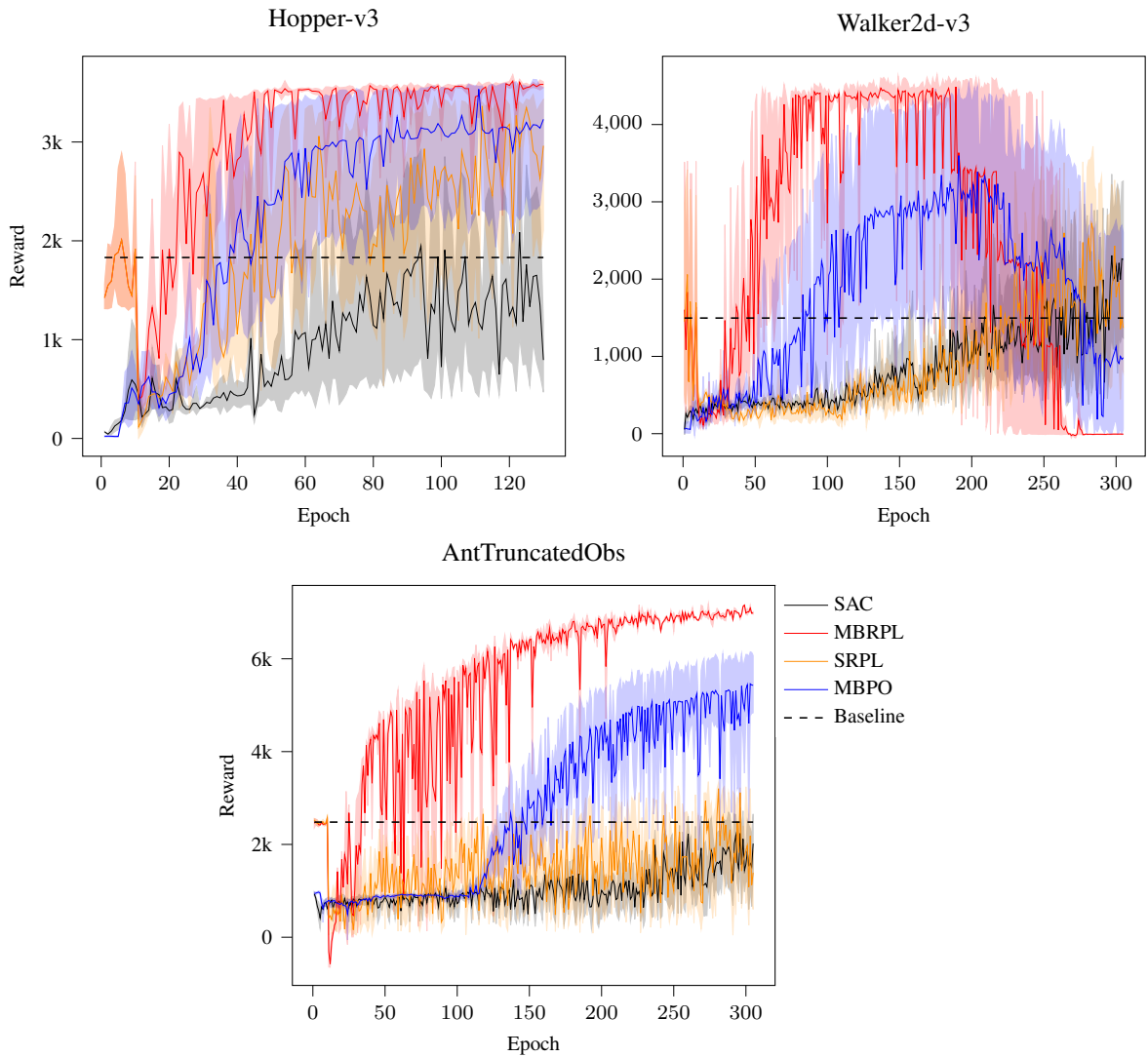

    \centering
    \input{figures/mujoco_exp/hopper_baselines_comparison}
    \input{figures/mujoco_exp/walker_baselines_comparison}
    \input{figures/mujoco_exp/ant_baselines_comparison}
    \caption{Comparison of MBRPL against the benchmark SAC, and the ablation RSAC on three MuJoCo tasks. MBPO plays both the role of benchmark and ablation.}
    \label{fig:all_comp_mujoco}
\end{figure*}
We compare against MBPO, SAC and SRPL. MBPO and SAC represent the state-of-the-art within model-based and model-free methods, respectively. SRPL is an ablation of our method (see the section on the antenna tilt problem for more information). 

\paragraph{Baseline policy.} The baseline policy is a SAC agent trained on a version of the environment with lower gravity. It is used by MBRPL and SRPL. We combined the residual policy with the mean of the baseline policy.

\paragraph{Model-based approach.}Initially, we used a simplified model-based approach to train the residual policy (see the antenna experiment section for more details). We found that the model was unable to learn an accurate enough representation of the environment, resulting in a collapse of the training process. Hence, we decided to use MBPO to train the residual correction term in MBRPL. The transition model in MBPO uses an ensemble of neural networks, which reduces epistemic uncertainty caused by the random initialization of the model weights. In addition, it is trained on all the so-far collected environment data until the validation loss increases each $k$:th step. On the antenna tilt environment, we only trained on one batch of data each time step, leading to far fewer model updates. The complexity of the MuJoCo tasks requires putting more effort on the model learning. Using an ensemble trained on all the current environment data yielded models of higher quality.
\paragraph{Training.} We trained the model-based methods for 130 epochs on Hopper, and 300 epochs on Walker and Truncated Ant. SAC was trained for 1000 and 3000 epochs, respectively. Each epoch consisted of 1000 environment steps. Each experiment was repeated 5 times. We used the MBPO and SAC implementations from MBRL-Lib~\cite{mbrllib}. On Hopper, we used the official MBPO parameters. On Truncated Ant, we had to use the parameters from the MBRL-Lib repository to achieve a similar performance as in the MBPO paper. In the Walker environment, the MBRL-lib parameters gave the better result, but led to a significantly lower maximum reward ($25\%$ lower) than the official MBPO implementation, and a decline after around 200 epochs. Our method, MBRPL, used the same parameters as those used for MBPO. It also requires some additional hyperparameters, such as choice of baseline policy and critic burn-in. The parameters can be found in the appendix. 

\subsection*{C.2. Results}

A comparison of MBRPL to the baselines can be seen in \cref{fig:all_comp_mujoco}. MBRPL was the most sample efficient method on most tasks. The difference  in efficiency is especially striking in the Truncated Ant environment, where MBRPL reached a reward of 7000 in 300 epochs, compared to the 5400 of MBPO.
MBRPL again outperforms MBPO in the Hopper environment, but the difference in reward when MBRPL converges is smaller. It appears as if it converges to a $\approx 10\%$ higher reward, and does so in a third of the epochs. 
In the Walker environment, MBRPL initially vastly outperforms MBPO in terms of sample efficiency. However, the two model-based methods show a decline in performance after around 200 epochs. A possible explanation is a catastrophic forgetting of the policy, or the loss being unstable, as was observed in the Humanoid environment in the MBRL-lib paper.

Another important thing to note is the  dip in performance of MBRPL and SRPL during the early stages of the training. It is possible that a larger critic burn-in could have prevented this dip in performance, as was found for the antenna tilt problem. The performance of the baselines in the training and true environment can be seen in \cref{tab:baseline_mujoco}.

Since the only difference between MBPO and this version of MBRPL is the use of the baseline policy, this must be the cause of the increased sample efficiency. By guiding the policy into higher reward regions earlier, it is able to accelerate learning. Interestingly, the model-based component of MBRPL also plays a crucial role - the model-free ablation RSAC is noticeably slower to converge than MBPO. The difference is the most significant on Truncated Ant. 

\begin{table}[h]
\centering
\begin{tabular}{lSSS}
\toprule
Task                       & \text{Hopper-v3} & \text{Walker2d-v3} & \text{AntTruncatedObs} \\ 
\midrule
Trained with $g$:            & \SI{7}{\meter\per\second\squared}  & \SI{6}{\meter\per\second\squared}    & \SI{4}{\meter\per\second\squared} \\ 
Reward in training env     & 3505 \pm 30 & 3150 \pm 1546 & 4793 \pm 1431 \\ 
Reward with \SI{9.81}{\meter\per\second\squared} & 1831 \pm 780 & 1496 \pm 1533 & 2481 \pm 70 \\
\bottomrule
\end{tabular}
\caption{The mean and 1 standard deviation of the reward collected by the baseline policies used for each task in the MuJoCo environment.}
\label{tab:baseline_mujoco}
\end{table}

\section*{Appendix D.}

This appendix contains the hyperparameters of the algorithms for all the experiments. 

\subsection*{D.1. Antenna environment}

\begin{minipage}{0.45\textwidth}
\begin{table}[H]
\centering
    \begin{tabular}{lM}
    \toprule
    \textbf{Parameter}    & \textbf{Value}                  \\ 
    \midrule
    Actor hidden layers   & [64, 64, 64]               \\ 
    Critic hidden layers  & [64, 64, 64]               \\ 
    Actor lr              & 3 \cdot 10^{-4}                    \\ 
    Critic lr             & 3 \cdot 10^{-4}                        \\ 
    Buffer size           & 10000                      \\ 
    Policy distribution   & \text{Tanh squashed Gaussian} \\ 
    Exploration type   & \text{Sample from policy} \\ 
    Random time steps      & 0                      \\ 
    Learning starts time step       & 100                        \\ 
    $\gamma$              & 0.9                       \\ 
    $\tau$                & 5\cdot 10^{-3}                      \\ 
    Target network update & \text{Every other time step}             \\ 
    Batch size            & 128                         \\ 
    Dueling Q             & \text{No}                         \\ 
    Entropy lr            & 3 \cdot 10^{-4}               \\ 
    $\alpha_0$              & 1                      \\ 
    Target entropy        & -1    \\ 
    Metric smoothing      & 5 \text{ episodes}              \\ 
    Time steps per epoch   & 500                        \\ 
    Number of epochs      & 20                          \\ 
    \bottomrule
    \end{tabular}
    \caption{Hyperparameters for SAC in the antenna     environment.}\label{tab:sac_params_antenna}
\end{table}
\end{minipage}\hfill
\begin{minipage}{0.45\textwidth}
\begin{table}[H]
\centering
\begin{tabular}{lM}
\toprule
\textbf{Parameter}    & \textbf{Value}            \\ 
\midrule
Actor hidden layers   & [64, 64, 64]               \\ 
Critic hidden layers  & {[}64, 64, 64{]}               \\ 
Actor lr          & 3 \cdot 10^{-4}                    \\ 
Critic lr            & 3 \cdot 10^{-4}                        \\ 
Buffer size           & 10000                      \\ 
Policy distribution   & \text{Deterministic}              \\ 
Exploration type            & \text{Epsilon-greedy}                   \\ 
Exploration $\epsilon_0$      & 1                        \\ 
Exploration $\epsilon_{final}$& 0.01                     \\ 
Exploration decay time steps & 3000                       \\ 
Random time steps      & 0                       \\ 
Learning starts time step       & 100                        \\ 
$\gamma$              & 0.9                       \\ 
$\tau$                & 5 \cdot 10^{-3}                        \\ 
Batch size            & 128                         \\ 
Dueling Q             & No                         \\ 
Metric smoothing      & 5 \text{ episodes}            \\ 
Time steps per epoch   & 500                        \\ 
Number of epochs      & 20                          \\ 
\bottomrule
\end{tabular}
\caption{Hyperparameters used for DQN in the antenna environment.}\label{tab:dqn_params_antenna}
\end{table}

\end{minipage}

\begin{table}[H]
\centering
\begin{tabular}{lM}
\toprule
\textbf{Parameter}             & \textbf{Value}  \\ 
\midrule
Transition model hidden layers & [64, 64, 64]   \\ 
Model lr                       & 10^{-3}          \\ 
Rollout length                 & 10             \\ 
Model batch size              & 128 \\
Model rollouts               & 1280 \\
\bottomrule
\end{tabular}
\caption{Additional hyperparameters used for the model-based component in the antenna environment.}\label{tab:model_based_params_antenna}
\end{table}

\subsection*{D.2. MuJoCo environments}
The parameters used by MBPO, SAC and MBRPL for Hopper, Walker and Truncated Ant can be seen below. The corrected policy $\pi_{\phi}^b$ was created by adding the mean of the baseline policy's output .

\begin{minipage}{0.45\textwidth}
\begin{table}[H]
\centering
\begin{tabular}{ll}
\toprule
\textbf{Parameter}         & \textbf{Value}        \\ 
\midrule
Training steps             & $125000$          \\ 
Epoch length               & $1000$                 \\ 
Initial exploration steps  & $5000$                 \\ 
Model lr                   & $10^{-3}$       \\ 
Model L2                   & $10^{-5}$       \\ 
Model batch size           & $256$                  \\ 
Model validation ratio    & $0.2$                  \\ 
Model hidden layers       & $[200, 200, 200, 200]$  \\ 
Ensemble size              & 7                    \\ 
Freq. train model           & $250$                  \\ 
Model rollouts per steps   & $400$                  \\ 
Rollout schedule           & $[20, 150, 1, 15]$     \\ 
SAC updates per step       & $20$                   \\ 
SAC $\gamma$               & $0.99$                 \\ 
SAC $\tau$                 & $5\cdot 10^{-3}$       \\ 
SAC $\alpha_0$             & $1$                  \\ 
SAC policy                 & Gaussian               \\ 
SAC target update interval & $1$                   \\ 
SAC auto entropy tuning    & True                  \\ 
SAC hidden layers          & $[256, 256]$         \\ 
SAC lr                     & $3\cdot10^{-4}$              \\ 
SAC batch size             & $256$                  \\ 
\bottomrule
\end{tabular}
\caption{Hyperparameters on Hopper.}
\end{table}
\end{minipage}\hfill
\begin{minipage}{0.45\textwidth}
\begin{table}[H]
\begin{tabular}{ll}
\toprule
\textbf{Parameter}         & \textbf{Value}        \\ 
\midrule
Training steps             & $3\cdot 10^5$          \\ 
Epoch length               & $1000$                 \\ 
Initial exploration steps  & $5000$                 \\ 
Model lr                   & $10^{-3}$       \\ 
Model L2                   & $10^{-5}$       \\ 
Model batch size           & $256$                  \\ 
Model validation ratio    & $0.2$                  \\ 
Model hidden layers       & $[200, 200, 200, 200]$ \\ 
Ensemble size              & $7$                    \\ 
Freq. train model           & $250$                  \\ 
Model rollouts per steps   & $400$                  \\ 
Rollout schedule           & $[20, 150, 1, 1]$     \\ 
SAC updates per step       & $20$                   \\ 
SAC $\gamma$               & $0.99$                 \\ 
SAC $\tau$                 & $5\cdot 10^{-3}$       \\ 
SAC $\alpha_0$             & $0.2$                  \\ 
SAC policy                 & Gaussian               \\ 
SAC target update interval & $4$                    \\ 
SAC auto entropy tuning    & False                  \\ 
SAC hidden layers          & $[1024, 1024]$         \\ 
SAC lr                     & $10^{-4}$              \\ 
SAC batch size             & $256$                  \\ 
\bottomrule
\end{tabular}
\caption{Hyperparameters on Walker2d-v3}
\end{table}
\end{minipage}

\begin{table}[H]
\centering
\begin{tabular}{ll}
\toprule
\textbf{Parameter}         & \textbf{Value}        \\ 
\midrule
Training steps             & $3\cdot 10^5$          \\ 
Epoch length               & $1000$                 \\ 
Initial exploration steps  & $5000$                 \\ 
Model lr                   & $3\cdot 10^{-4}$       \\ 
Model L2                   & $5\cdot 10^{-5}$       \\ 
Model batch size           & $256$                  \\ 
Model validation ratio    & $0.2$                  \\ 
Model hidden layers      & $[200, 200, 200, 200]$ \\ 
Ensemble size              & $7$                    \\ 
Freq. train model           & $250$                  \\ 
Model rollouts per steps   & $400$                  \\ 
Rollout schedule           & $[20, 100, 1, 25]$     \\ 
SAC updates per step       & $20$                   \\ 
SAC $\gamma$               & $0.99$                 \\ 
SAC $\tau$                 & $5\cdot 10^{-3}$       \\ 
SAC $\alpha_0$             & $0.2$                  \\ 
SAC policy                 & Gaussian               \\ 
SAC target update interval & $4$                    \\ 
SAC auto entropy tuning    & False                  \\ 
SAC hidden layers          & $[1024, 1024]$         \\ 
SAC lr                     & $10^{-4}$              \\ 
SAC batch size             & $256$                  \\ 
\bottomrule
\end{tabular}
\caption{Hyperparameters on truncated ant.}
\end{table}\label{sec:appendix}

\end{document}